\PassOptionsToPackage{numbers,sort&compress}{natbib}          
\PassOptionsToPackage{ruled,vlined,linesnumbered}{algorithm2e} 

\documentclass{article}

\usepackage[preprint]{neurips_2024}

\usepackage{amsmath,amssymb,amsfonts}
\usepackage{amsthm}

\usepackage{graphicx}
\usepackage{subcaption}
\usepackage{booktabs}

\usepackage{microtype}
\usepackage{nicefrac}
\usepackage{xcolor}
\usepackage{fancyvrb}

\usepackage{tikz}
\usetikzlibrary{positioning,arrows.meta,calc,decorations.pathreplacing,fit}

\usepackage{hyperref}   
\hypersetup{
  colorlinks=true,
  linkcolor=red,       
  citecolor=blue,      
  urlcolor=magenta     
}

\usepackage{algorithm2e}
\DontPrintSemicolon
\SetKwInput{KwIn}{Input}
\SetKwInput{KwOut}{Output}
\SetKwProg{Fn}{Function}{}{end}

\DeclareMathOperator*{\argmin}{arg\,min}

\theoremstyle{plain}
\newtheorem{theorem}{Theorem}[section]
\newtheorem{lemma}[theorem]{Lemma}
\newtheorem{proposition}[theorem]{Proposition}

\theoremstyle{definition}

\theoremstyle{remark}

\title{Mental Accounts for Actions: EWA-Inspired Attention in Decision Transformers}

\author{%
  Zahra Aref \\
  Department of Electrical and Computer Engineering \\
  Rutgers University\\
  New Brunswick, NJ 08901 \\
  \texttt{zahra.aref@rutgers.edu} \\
  \And
  Narayan B. Mandayam \\
  WINLAB, Department of ECE \\
Rutgers University \\
  New Brunswick, NJ 08901 \\
  \texttt{narayan@winlab.rutgers.edu} \\
}

\begin{document}

\maketitle

\begin{abstract}
Transformers have emerged as a compelling architecture for sequential decision-making by modeling trajectories via self-attention. In reinforcement learning (RL), they enable return-conditioned control without relying on value function approximation. Decision Transformers (DTs) exploit this by casting RL as supervised sequence modeling, but they are restricted to offline data and lack exploration. Online Decision Transformers (ODTs) address this limitation through entropy-regularized training on on-policy rollouts, offering a stable alternative to traditional RL methods like Soft Actor-Critic, which depend on bootstrapped targets and reward shaping.
Despite these advantages, ODTs use standard attention, which lacks explicit memory of action-specific outcomes. This leads to inefficiencies in learning long-term action effectiveness. Inspired by cognitive models such as Experience–Weighted Attraction (EWA), we propose Experience-Weighted Attraction with Vector Quantization for Online Decision
Transformers (EWA–VQ–ODT), a lightweight module that maintains per-action “mental accounts” summarizing recent successes and failures. Continuous actions are routed via direct grid lookup to a compact vector-quantized codebook, where each code stores a scalar attraction updated online through decay and reward-based reinforcement. These attractions modulate attention by biasing the columns associated with action tokens, requiring no change to the backbone or training objective.
On standard continuous-control benchmarks, EWA–VQ–ODT improves sample efficiency and average return over ODT, particularly in early training. The module is computationally efficient, interpretable via per-code traces, and supported by theoretical guarantees that bound the attraction dynamics and its impact on attention drift.
\end{abstract}


\section{Introduction}
\label{sec:intro}
Transformers, originally developed for natural language processing, have emerged as a powerful backbone for sequential decision-making. Their ability to model long-range dependencies through self-attention makes them an appealing choice for reinforcement learning (RL), particularly when actions, states, and rewards can be treated as tokens in a trajectory sequence \citep{chen2021decision,janner2021offline,zheng2022online}.

While traditional RL methods such as Soft Actor-Critic (SAC) \citep{haarnoja2018soft} rely on value estimation and temporal-difference (TD) bootstrapping, Decision Transformers (DTs) bypass value functions altogether by predicting actions through return-conditioned sequence modeling. However, DTs are trained offline and lack mechanisms for online adaptation or exploration, which limits their performance in dynamic or partially observed environments.

Online Decision Transformers (ODT) address this limitation by adapting policies online while maintaining the supervised learning objective of DTs. ODTs blend imitation and exploration by maximizing the log-likelihood of actions while conditioning on returns. This formulation avoids common RL challenges like unstable TD targets and handcrafted reward shaping. Yet, ODTs still operate with standard attention mechanisms, which lack explicit memory for tracking the long-term effectiveness of past actions.

Inspired by cognitive theories such as Experience–Weighted Attraction (EWA) \citep{camerer2003behavioral}, we posit that agents should maintain per-action “mental accounts” — cumulative records of action success — that decay over time and guide attention. Unlike standard transformers, humans exhibit persistence and bias in decision-making, often favoring actions that have historically performed well even without counterfactual modeling.

\paragraph{Contributions.}
\begin{itemize}
  \item We propose Experience-Weighted Attraction with Vector Quantization for Online Decision Transformers (EWA–VQ–ODT), a lightweight module that maintains per–action (code) attractions online and injects them as a small bias on action-token attention columns in DT/ODT.
  \item We instantiate a fast vector–quantized routing scheme for continuous actions (fixed codebook, direct grid lookup), enabling constant–time mapping and batched GPU updates.
  \item We provide empirical support for this design. On D4RL benchmark tasks, our method improves average returns and early training sample efficiency over standard ODT.
  \item We conduct a theoretical study, formalizing attraction dynamics as a bounded, exponentially weighted sum. We prove that the attention drift caused by attraction bias is safely limited in total variation.
\end{itemize}

\section{Related Work}
\label{sec:related}

\subsection{Transformers for Reinforcement Learning}
\label{subsec:rw-transformers}

A growing line of work recasts reinforcement learning as sequence modeling. Early evidence showed that high--capacity sequence predictors can model trajectories and even support planning with generic decoding procedures \citep{janner2021offline}. Building on this view, \citet{chen2021decision} introduced the \emph{Decision Transformer} (DT), which conditions a causal Transformer on desired return, past states, and actions to generate actions autoregressively, matching strong offline RL baselines without explicit value functions. \citet{zheng2022online} extended this paradigm to interactive settings with the \emph{Online Decision Transformer} (ODT), blending offline pretraining and online finetuning via sequence--level entropy regularization for sample--efficient exploration.

Applying Transformers in RL also raises stability and credit--assignment challenges. \citet{parisotto2020stabilizing} identified optimization issues of standard Transformers under RL objectives and proposed GTrXL, architectural modifications that improve stability and performance across memory--intensive tasks. Complementary advances address long--horizon credit propagation; for example, \citet{malkin2022trajectory} introduced trajectory balance objectives in GFlowNets to propagate credit more efficiently along action sequences. Other works combine sequence modeling with value--based structure: \citet{hu2024q} regularize the Transformer with learned $Q$--values to better stitch optimal segments from suboptimal data in offline RL.

Scaling trends in control further support the utility of Transformer policies. RT--1 \citep{brohan2022rt} trains a robotics Transformer on large, diverse real--world data and demonstrates broad generalization across tasks, indicating that high--capacity sequence models can absorb heterogeneous experience in control domains.

Our work targets a complementary axis: \emph{explicit, action--specific memory}. Rather than modifying architectures for stability \citep{parisotto2020stabilizing} or augmenting objectives with values \citep{hu2024q}, we inject a lightweight, EWA--inspired per--action attraction signal into attention columns of DT/ODT. This preserves the return--conditioning benefits of sequence modeling \citep{chen2021decision,zheng2022online} while providing a persistent, decaying ``mental account'' of action outcomes that improves credit flow with minimal overhead and no counterfactual modeling.

\subsection{Cognitively Grounded RL and Experience--Weighted Attraction}
\label{subsec:rw-cognitive-ewa}

Behavioral and cognitive theories provide formal accounts of how humans accumulate and use action--specific experience. \citet{camerer2003behavioral} synthesizes extensive experimental evidence into behavioral game theory, emphasizing learning dynamics and limited strategic depth; a central component is \emph{experience--weighted attraction} (EWA), where per--action ``attractions'' aggregate outcomes with depreciation, predicting regularities in repeated interactive settings. Cognitive architectures such as ACT--R \citep{anderson2009can} instantiate persistent declarative and procedural memories with activation and decay, while instance--based learning (IBL) \citep{gonzalez2003instance} explains dynamic choice by retrieving and blending past situation--action--utility instances with recency and utility weighting. These strands motivate maintaining explicit, decaying, action--specific memories---the conceptual basis for our ``mental accounts'' of success and failure.

Cognitively informed RL has also been explored in decision settings that require interpretability and robustness, including human–AI teaming for security. Cognitive–hierarchy models of bounded rationality \citep{camerer2004cognitive} and prospect–theoretic accounts of subjective valuation under uncertainty \citep{kahneman2013prospect} have been instantiated in deep RL for cloud defense and analyzed for their impact on policy learning \citep{aref2025human,aref2022impact}. In web-based human-in-the-loop DRL experiments on Amazon Mechanical Turk, \citet{aref2025human} report decision patterns consistent with Prospect Theory (PT) and Cumulative Prospect Theory (CPT): participants tend to avoid re-selection after failure but persist after success, reflecting heightened loss sensitivity and biased probability weighting (underestimating gains after loss and overestimating continued success). This asymmetry is also aligned with EWA’s mechanism—decaying attractions for unchosen/unsuccessful actions and reinforcement of attractions after successful choices. Collectively, these studies underscore the value of explicit, human-interpretable biases and persistent memory mechanisms when actions must be justified and robust to nonstationarity.

Within RL, several mechanisms echo these principles by adding fast memory or sharpening credit assignment. Neural Episodic Control \citep{pritzel2017neural} augments deep agents with a semi--tabular episodic store for rapidly updated value estimates, enabling swift assimilation of new experience under sparse reward. RUDDER \citep{arjona2019rudder} tackles delayed credit by redistributing returns and decomposing contributions, converting long--horizon credit into a supervised regression problem. Resource--rational RL \citep{binz2022modeling} models human exploration as bounded optimality under computational costs, supporting lightweight memory and priors as plausible inductive biases. In complementary directions, preference--based RL and human feedback pipelines---from policy shaping with direct human labels \citep{griffith2013policy} and learning from pairwise preferences \citep{christiano2017deep} to large--scale RLHF \citep{bai2022training}, fine--grained segment--level rewards \citep{wu2023fine}, few--shot preference models \citep{hejna2023few}, and personalized latent preference learning \citep{poddar2024personalizing}---demonstrate that simple, interpretable supervisory signals can reliably steer complex policies. Robust adversarial model--based offline RL \citep{rigter2022rambo} addresses stability under distribution shift and is orthogonal to cognitive inductive biases. Closer to sequence models, hierarchical/in--context formulations \citep{huang2024context} reduce decision horizons by structuring control, suggesting a productive interface between cognitive principles and transformers.

Our contribution differs in both target and mechanism. Rather than learning external reward models or episodic memories, we integrate a \emph{per--action, decaying attraction} directly into the transformer's attention via a small bias on action--token columns. This EWA--inspired signal is vector--quantized for continuous actions, learned online from observed rewards (positive and negative), and requires no counterfactual modeling or architectural changes. In effect, we endow Decision/Online Decision Transformers with an explicit, interpretable memory of action outcomes that persistently shapes credit flow, complementing prior cognitive and human--feedback approaches.

\subsection{Action Representation and Quantization}
\label{subsec:rw-quantization}

Representing continuous actions with compact discrete surrogates can simplify learning, improve stability, and enable structured inductive biases. A prominent line of work learns discrete latent codes with vector quantization. \citet{Oord2017NeuralDR} introduced VQ--VAE, which pairs an autoregressive prior with a learned codebook to obtain discrete latents that avoid posterior collapse and support high–fidelity generation. Building on this idea for control, \citet{luo2023action} propose state–conditioned action quantization via VQ--VAE in offline RL, showing that discretizing the action space allows conservative objectives (e.g., IQL, CQL, BRAC) to be applied more precisely, yielding substantial gains on robotics benchmarks.

Beyond direct action discretization, several strands learn \emph{action abstractions} or latent representations that can serve as surrogates. Unsupervised skill discovery methods such as DIAYN \citep{eysenbach2018diversity} maximize diversity to produce a repertoire of discrete skills that can be composed or fine–tuned for downstream tasks, effectively quantizing behavior space. Latent state--space approaches, e.g., DeepMDP \citep{gelada2019deepmdp}, learn compact predictive representations that facilitate planning and policy learning; although continuous, these latents can be clustered or discretized when needed to provide decision tokens.

Kernelized clustering provides an alternative, \emph{soft} form of quantization in which assignments are driven by similarities in a reproducing kernel space. Kernel $k$--means and its spectral counterparts \citep{dhillon2004kernel} (and soft variants via fuzzy clustering) enable similarity–aware grouping without explicit parametric codebooks, at the cost of potentially higher per–step assignment overhead.

We target a lightweight, \emph{online} surrogate representation compatible with transformer policies. Instead of learning a state–conditioned codebook \citep{luo2023action} or maintaining kernel similarities \citep{dhillon2004kernel}, we use a fixed, grid–derived codebook with direct lookup to achieve constant–time routing. On these discrete surrogates, we maintain an EWA–style per–action memory and inject it as a small bias into attention. This retains many practical benefits of discretization (stable routing, small memory) while enabling an explicit, cognitively motivated mechanism for credit flow over actions.

\section{Preliminaries}
\label{sec:prelim}

\paragraph{Transformer Architecture.}
Transformers are powerful sequence models that leverage multi-head self-attention to capture long-range dependencies across input sequences \citep{vaswani2017attention}. In the causal (decoder-only) transformer used for control tasks, a sequence of input tokens \( x_1, x_2, \ldots, x_T \) is mapped into query, key, and value vectors via learned weight matrices: \( Q_t = x_t W^Q, K_t = x_t W^K, V_t = x_t W^V \). Attention weights are then computed as
\begin{equation}
\text{Attention}(Q,K,V) = \text{softmax}\left(\frac{QK^\top}{\sqrt{d_k}}\right)V,
\end{equation}
with the attention mask enforcing causality by preventing each token from attending to future positions.

Transformer models serve as the foundation for Decision Transformers, which repurpose this architecture for reinforcement learning by modeling return-conditioned trajectories.

\paragraph{Reinforcement learning.}
Standard reinforcement learning (RL) formalizes sequential decision-making as the solution to a Markov Decision Process (MDP) \((\mathcal{S}, \mathcal{A}, P, r, \gamma)\) \citep{sutton2018reinforcement}. At time \(t\), the agent observes state \(s_t\), selects an action \(a_t\), receives a reward \(r_t = r(s_t, a_t)\), and transitions to a new state \(s_{t+1} \sim P(\cdot \mid s_t, a_t)\). The goal is to find a policy \(\pi(a_t \mid s_t)\) that maximizes the expected discounted return:
\begin{equation}
\max_\pi \; \mathbb{E}_\pi\left[\sum_{k=0}^{\infty} \gamma^k r_{t+k} \right].   
\end{equation}

While classical RL methods rely on value functions, Q-learning, or policy gradients, recent advances have introduced transformer-based models that treat trajectory optimization as a sequence prediction task, thereby enabling offline learning from demonstration data.

\paragraph{Decision Transformers.}
Decision Transformers (DT; \citet{chen2021decision}) reframe policy learning as a supervised sequence modeling problem. A trajectory is represented as a sequence of return-to-go \(g_t\), state \(s_t\), and action \(a_t\) tokens, i.e., \(\tau = (g_1, s_1, a_1, \ldots, g_K, s_K, a_K)\), where \(g_t = \sum_{k=t}^K \gamma^{k - t} r_k\). Here, $K$ is a hyperparameter denoting the context length for the transformer. These tokens are embedded and passed through a causal transformer that autoregressively predicts the action \(a_t\) based on prior tokens. 

By conditioning on desired returns, DT enables offline imitation of expert behavior. However, it lacks support for online interaction and cannot adapt to changes in environment dynamics or reward structures.
Online Decision Transformers extend this formulation to interactive, online settings by enabling real-time data collection and adaptation.

\paragraph{Online Decision Transformer.}

Online Decision Transformer (ODT; \citet{zheng2022online}) extends Decision Transformers to interactive online reinforcement learning, enabling agents to collect fresh experience and adapt continuously to nonstationary environments. Unlike DT, which trains solely on fixed offline datasets, ODT leverages newly acquired trajectories to improve its policy in real time.

At each timestep $t$, the ODT agent observes a state $s_t$, takes an action $a_t$, receives a reward $r_t$, and updates the return-to-go $g_t$. Sequences of triplets $(g_t, s_t, a_t)$ are passed to a causal transformer trained to predict the next action conditioned on prior tokens.

The policy is modeled as a stochastic distribution over actions, typically a multivariate Gaussian:

\begin{equation}
    \pi_\theta(a_t \mid s_{-K, t}, g_{-K, t}) = \mathcal{N} \big( \mu_\theta(s_{-K, t}, g_{-K, t}), \Sigma_\theta(s_{-K, t}, g_{-K, t}) \big),
\end{equation}
where $\Sigma_\theta$ is assumed to be diagonal. Training minimizes the negative log-likelihood (NLL) of observed actions over $K$ steps:

\begin{equation}
    \mathcal{J}(\theta) = \frac{1}{K} \mathbb{E}_{(a, s, g) \sim \mathcal{T}} \left[ \sum_{k=1}^{K} - \log \pi_\theta(a_k \mid s_{-K, k}, g_{-K, k}) \right].
\end{equation}

To promote exploration during online finetuning, ODT imposes a sequence-level entropy constraint on the policy:

\begin{equation}
    \mathcal{H}_{\theta}^{\tau}[r(a \mid s, g)] = \frac{1}{K} \, \mathbb{E}_{(s, g) \sim \mathcal{T}} \left[ \sum_{k=1}^{K} \mathcal{H} \left[ \pi_\theta(a_k \mid s_{-K, k}, g_{-K, k}) \right] \right] \geq \beta,
\end{equation}

where $\mathcal{H} \left[ \pi_\theta(a_k) \right]$ denotes the Shannon entropy of distribution $\pi_\theta(a_k)$, $\beta$ is a prefixed minimum entropy threshold, and $\mathcal{T}$ denotes the data distribution (offline or replay-buffer based).

This gives rise to the following constrained optimization problem:

\begin{equation}
    \min_{\theta} \mathcal{J}(\theta) \quad \text{subject to} \quad \mathcal{H}_{\theta}^{\tau}[r(a \mid s, g)] \geq \beta.
\end{equation}

To solve this, ODT formulates a dual problem using the Lagrangian:

\begin{equation}
    \mathcal{L}(\theta, \lambda) = \mathcal{J}(\theta) + \lambda (\beta - \mathcal{H}_{\theta}^{\tau}[r(a \mid s, g)]),
\end{equation}

and optimizes it via alternating gradient updates over $\theta$ and the dual variable $\lambda \geq 0$.

Unlike soft actor-critic (SAC; \citet{haarnoja2018soft}), where the objective is based on return maximization, ODT retains the supervised NLL form. The entropy term $\mathcal{H}_{\theta}^{\tau}$ can be interpreted as a cross-entropy during offline training (due to data mismatch) and converges to true entropy during online adaptation, thereby bridging offline imitation and online exploration in a unified transformer-based framework.

Compared to Decision Transformer (DT), ODT supports online learning by actively adapting to newly collected trajectories rather than relying solely on static datasets. This allows for policy improvement over time in dynamic environments.

Compared to traditional reinforcement learning methods, ODT eliminates the need for value function approximation, temporal-difference (TD) bootstrapping, and handcrafted reward shaping. Classical RL algorithms often suffer from instability due to value estimation errors and delayed credit assignment. ODT bypasses these challenges by directly modeling the return-conditioned action distribution, leading to more stable training and better performance in sparse-reward or high-variance environments.

Despite these advantages, ODT does not explicitly account for cognitive dynamics such as regret, action inertia, or memory-based attention—factors known to influence human decision-making. To address this gap, we integrate the Experience–Weighted Attraction (EWA) model into the ODT framework. This cognitive extension introduces structured biases into attention and memory, allowing us to simulate and study more human-like sequential decision behavior.

\paragraph{Experience--Weighted Attraction (EWA).}
EWA \citep{camerer1999experience} is a bounded-rationality model from behavioral game theory that maintains, for each action \(j\), an \emph{attraction} \(A_j(t)\) that aggregates past payoff evidence with exponential decay. Its canonical update is given by:
\begin{equation}
    A_j(t) =
\frac{(1-\phi) N(t-1) A_j(t-1) + \left[\delta + (1-\delta)\,\mathbf{1}\{j = j_t\}\right] \pi_j(t)}{N(t)},
\qquad
N(t) = (1 - \rho) N(t - 1) + 1,
    \label{eq:ewa-canonical}
\end{equation}
where \(\phi \in (0,1)\) depreciates attractions, \(\rho \in (0,1)\) depreciates experience, and \(\delta \in [0,1]\) balances realized and foregone payoffs. Attractions persist for all actions, decaying over time even when unchosen.

We integrate EWA-inspired attention into ODT to simulate human-like decision patterns that emphasize inertia, long-term memory, and differential treatment of past actions. This allows us to investigate how cognitive biases may improve robustness and interpretability in transformer-based control.

\section{Method}

\paragraph{Overview.}
Figure~\ref{fig:ewa_vq_odt} summarizes the proposed EWA--VQ--ODT framework.
\textbf{(A)} A length-$K$ action context is maintained, and the current action $a_t$ is routed by a grid lookup to a discrete code $j_t$ in a fixed vector-quantized codebook.
\textbf{(B)} Each code $\ell$ holds an attraction $A_\ell(t)$ updated online by a simplified EWA rule with global decay $(1-\phi)$ and reinforcement $\delta\,\tilde r_t$ \emph{only} for the chosen code $j_t$ (Eq.~\eqref{eq:ewa-simplified}); the up arrow marks reinforcement for $j_t$ while dashed down arrows indicate decay on others.
\textbf{(C)} The attraction vector $A$ is scaled to a column bias $B=\beta A$ and added to the pre-softmax attention logits \emph{only} on action-token columns (Eq.~\eqref{eq:attn-bias}), gently steering attention toward actions with higher recent attraction without changing the backbone.

\begin{figure}[t]
  \centering
  \includegraphics[width=\linewidth]{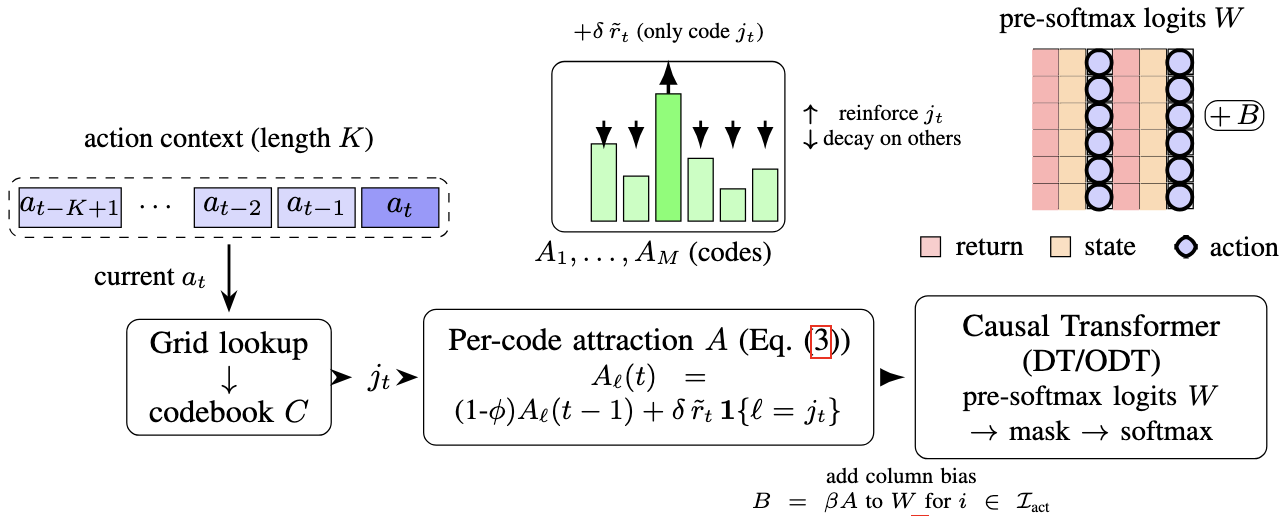}
  \caption{\textbf{EWA--VQ--ODT overview.} (A) The last $K$ actions form the context; the current action $a_t$ is routed by a grid lookup (cell$\!\to\!$code table $T$) to code $j_t$. 
(B) Per-code attractions implement a recency-weighted memory via decay $(1-\phi)$ and reinforcement $\delta\,\tilde r_t$ on the chosen code (Eq.~\eqref{eq:ewa-simplified}). 
(C) Attractions add a pre-softmax \emph{column bias} $B=\beta A$ to the action-token \emph{columns} of the attention logits $W$ (Eq.~\eqref{eq:attn-bias}); highlighted heatmap columns indicate $i\in\mathcal I_{\text{act}}$.}
\label{fig:ewa_vq_odt}
\end{figure}

\subsection{Experience--Weighted Attraction (EWA) for Decision Transformers}
\label{subsec:ewa}

We endow a Decision/Online Decision Transformer with an explicit, per--action–surrogate memory by maintaining a scalar \emph{attraction} for each vector–quantized code (Sec.~\ref{subsec:vq-grid}) and updating it online. Let $j_t$ be the routed code for action $a_t$. 
Our EWA-inspired update is \begin{equation} \label{eq:ewa-simplified} A_\ell(t) \;=\; (1-\phi)\,A_\ell(t-1) \;+\; \delta\,\tilde r_t\,\mathbf{1}\{\ell=j_t\}, \end{equation}
with $\phi\!\in\!(0,1)$ (forgetting), $\delta\!\in\![0,1]$ (chosen–action weight), and $\tilde r_t$ a centered/clipped reward so that positive outcomes increase attraction (“success”) and negative outcomes decrease it (“failure”). Thus, previously observed but unchosen codes retain a decaying memory; the chosen code receives the incremental reinforcement. Unseen codes remain at zero until first routed. The resulting $A_{j_t}(t)$ is later injected as a small column–bias into action-token attention logits (Eq.~\eqref{eq:attn-bias}).

\paragraph{Why the simplified form?}
The canonical EWA (Sec.~\ref{sec:prelim}) normalizes by $N(t)$ and can include counterfactual payoffs. We omit both for (i) \textit{compatibility} with online sequence models (no counterfactual model required), (ii) \textit{stability/efficiency} (constant-time, vectorized updates without a running denominator), and (iii) \textit{calibration} (the attention bias absorbs scale). Empirically, Eq.~\eqref{eq:ewa-simplified} preserves the intended recency-weighted reinforcement signal while keeping the mechanism lightweight and easy to integrate.

\begin{proposition}[Closed form, boundedness, steady state]
Let the per-code attraction follow the simplified EWA update
\[
A_\ell(t) = (1-\phi) A_\ell(t-1) + \delta\,\tilde r_t\,\mathbf{1}\{\ell=j_t\}, 
\quad \phi \in (0,1), \quad |\tilde r_t|\le R .
\]
Then $A_\ell(t)$ admits a closed form, is uniformly bounded by 
$\delta R/\phi$, and under stationarity satisfies
\[
\lim_{t \to \infty} \, \mathbb{E}[A_\ell(t)] \;=\; \frac{\delta\,p_\ell\,\mu_\ell}{\phi}.
\]

\label{prop:ewa_closedform}
\end{proposition}

\begin{proof}[Proof sketch]
Unrolling the recursion yields the closed form and boundedness via a geometric sum.  
Under stationarity, the expectation satisfies a linear recursion with fixed point $(\delta p_\ell \mu_\ell)/\phi$.  
Full details are deferred to Appendix~\ref{app:ewa_proof}. \qedhere
\end{proof}

The proposition \ref{prop:ewa_closedform} implies that $A_\ell(t)$ is an exponentially weighted sum of realized rewards for code~$\ell$, with weight $(1-\phi)^{\,t-\tau}$ on the outcome at time~$\tau$; older events thus contribute less. With zero initialization, the attraction is uniformly bounded as $
|A_\ell(t)| \;\le\; \frac{\delta R}{\phi}.
$
Under stationarity with $\Pr(j_t=\ell)=p_\ell$ and $\mathbb{E}[\tilde r_t\mid j_t=\ell]=\mu_\ell$, the long-run expectation converges as $
\lim_{t \to \infty} \, \mathbb{E}[A_\ell(t)] \;=\; \frac{\delta\,p_\ell\,\mu_\ell}{\phi},$
increasing with $\delta$ and $p_\ell\mu_\ell$ and decreasing with the forgetting rate~$\phi$.

\subsection{Vector--Quantized Action Surrogates with Direct Grid Lookup}
\label{subsec:vq-grid}

Continuous actions \(a_t\!\in[-1,1]^d\) are routed to \(M\) discrete \emph{codes} \(C=\{c_i\}_{i=1}^{M}\) via vector quantization. A fixed codebook is built once from an axis-aligned grid with \(b\) bins per dimension; the grid table \(\mathsf{T}:\{0,\ldots,b^d-1\}\!\to\!\{1,\ldots,M\}\) maps each cell to its nearest code. At runtime, \(a_t\) is quantized to a cell and the code index is read from \(\mathsf{T}\): \(j_t=\mathsf{T}(\text{cell}(a_t))\). Misses (rare) fall back to a small batched \(L_2\) search. The online and offline procedures appear in Alg.~\ref{alg:ewavq} and Alg.~\ref{alg:grid}.

\begin{algorithm}[H]
\caption{EWAVQ (Vector-Quantized EWA), per trajectory}
\label{alg:ewavq}
\KwIn{Actions $\{a_t\}_{t=1}^K$, rewards $\{r_t\}_{t=1}^K$, decay $\phi$, weight $\delta$, codebook $C=\{c_i\}_{i=0}^{N-1}$, grid bins $b$, table $T$}
\KwOut{$D_{\text{codebook}}$, $D_{\text{trajectory}}$}

$A \in \mathbb{R}^N \leftarrow \mathbf{0}$;\quad $D_{\text{trajectory}} \leftarrow [\,]$\;

\Fn{\textsc{Route}$(a)$}{
  Map $a \in [-1,1]^D$ to grid coords $q$ and linear index $\ell$;\;
  \If{$\ell$ is cached in $T$}{\Return $T[\ell]$}
  \Else{\Return $\argmin_{i \in \{0,\dots,N-1\}} \|a - c_i\|_2$}
}

\For{$t \gets 1$ \KwTo $K$}{
  $i^\star \leftarrow \textsc{Route}(a_t)$\;
  $A \leftarrow (1-\phi)\,A$ \tcp*[r]{global decay (vectorized)}
  $A[i^\star] \leftarrow A[i^\star] + \delta\, r_t$\;
  append $(a_t, r_t, i^\star, A[i^\star])$ to $D_{\text{trajectory}}$\;
}

$D_{\text{codebook}} \leftarrow \{(i, A[i]) \mid A[i] \neq 0\}$\;
\Return $D_{\text{codebook}}, D_{\text{trajectory}}$\;
\end{algorithm}

\begin{algorithm}[H]
\caption{Grid-based Action Quantization (cell $\to$ code)}
\label{alg:grid}
\KwIn{Dim $D$, requested codes $N_{\rm req}$, bins $b_{\rm req}$}
\KwOut{$T \in \{0,\dots,N{-}1\}^{b^D}$}

$b \leftarrow \min\!\big\{8,\,\max\!\{2,\,\lfloor N_{\rm req}^{1/D}\rfloor\}\big\}$;\quad
\While{$b^D > N_{\rm req}$ \textbf{and} $b>2$}{$b \leftarrow b-1$}
\If{$D \ge 6$ \textbf{and} $b^D > N_{\rm req}$}{$N \leftarrow \min(b^D,128)$}\Else{$N \leftarrow N_{\rm req}$}\;

\For{$i \gets 0$ \KwTo $N-1$}{
  Let $g_i \in \{0,\dots,b-1\}^D$ be the base-$b$ digits of $i$;\;
  Set code $c_i \in [-1,1]^D$ by $c_i[d] \leftarrow \tfrac{2\,g_i[d]}{b-1} - 1$ for $d=1..D$\;
}

Allocate $T$ of size $b^D$;\;
\For{cells $\ell \gets 0$ \KwTo $b^D{-}1$ \textbf{in batches}}{
  Map $\ell$ to $g_\ell$ (base-$b$) and $a_\ell[d] \leftarrow \tfrac{2\,g_\ell[d]}{b-1}-1$;\;
  $T[\ell] \leftarrow \argmin_{i \in \{0,\dots,N-1\}} \|a_\ell - c_i\|_2$ \tcp*[r]{vectorized on GPU}
}
Cache $T$ under key $(D,b,N,\text{env})$ and return\;
\end{algorithm}

\subsection{Injecting Attractions into Attention}
\label{subsec:ewa-attn}

Let \(W\in\mathbb{R}^{N\times H\times S\times S}\) denote pre-softmax attention logits over tokenized returns, states, and actions. Let \(\mathcal{I}_{\mathrm{act}}\subset\{1,\ldots,S\}\) be the action-token column indices and let \(t(i)\) be the timestep associated with action column \(i\). From \eqref{eq:ewa-simplified}, define the per-batch, per-head \emph{attraction bias} as
\begin{equation}
B_{n,h,:,i} \;=\; 
\begin{cases}
\beta\,A_{j_{t(i)}}(t(i))\,\mathbf{1}_S, & i\in\mathcal{I}_{\mathrm{act}},\\[3pt]
\mathbf{0}_S, & \text{otherwise},
\end{cases}
\qquad
\tilde W \;=\; W \;+\; B,
\label{eq:attn-bias}
\end{equation}
where \(\mathbf{1}_S\) is a length-\(S\) all-ones vector (broadcast along rows) and \(\beta>0\) is a small scalar (or per-head scale; see ablations). We then apply the usual causal/padding masks and softmax. Intuitively, \(\beta\,A_{j_{t(i)}}\) is a column-bias that increases (decreases) attention mass toward action tokens with stronger (weaker) accumulated attraction—an explicit, recency-weighted memory of success and failure.%
\footnote{We keep the attention architecture unchanged: the bias is additive, applied before masking/softmax.}

\begin{lemma}[Controlled attention drift under a bounded bias]
Let $p=\mathrm{softmax}(z)\in\Delta^n$ and $q=\mathrm{softmax}(z+b)$ for any bias $b$ with $\|b\|_\infty\le \varepsilon$.  
Then the total variation distance satisfies
\[
\|q-p\|_{\mathrm{TV}} \;\le\; \tanh(\varepsilon).
\]
\label{lem:softmax_bound}
\end{lemma}

\begin{proof}[Proof sketch]
The likelihood ratio $q_i/p_i$ is bounded in $[e^{-2\varepsilon}, e^{2\varepsilon}]$, which implies  
$\|q-p\|_{\mathrm{TV}} \le \tanh(\varepsilon)$.  
Full details are in Appendix~\ref{app:softmax_proof}. \qedhere
\end{proof}

Intuitively, perturbing logits by at most $\pm\varepsilon$ constrains each attention probability’s multiplicative change to $[e^{-2\varepsilon}, e^{2\varepsilon}]$; the worst case occurs when some entries get the full boost and others the full suppression, yielding the $\tanh(\varepsilon)$ bound. For small $\varepsilon$, $\tanh(\varepsilon)\!\approx\!\varepsilon$, so clipping $\|\beta A\|_\infty\le \varepsilon$ provides a simple, interpretable cap on attention drift.


\section{Experiments}
\label{sec:experiments}

\subsection{Setup}

\paragraph{Environments.}
We evaluate on two D4RL continuous-control benchmarks: \texttt{hopper-medium-v2} (action dim.\ $d{=}3$) and \texttt{walker2d-medium-replay-v2} ($d{=}6$). Datasets contain suboptimal offline trajectories; we fine-tune online following the standard offline-to-online protocol.

\paragraph{Baselines.}
We compare \textbf{EWA\textendash VQ\textendash ODT} to the \textbf{Online Decision Transformer (ODT)}. Architectures (backbone width/depth), tokenization, context length, optimizer, and training schedules are matched.

\paragraph{Training configuration.}
Unless noted: $10$ online iterations with evaluation every two iterations; batch size $64$; context length $K{=}20, 5$ for \texttt{hopper-medium-v2} and \texttt{walker2d-medium-replay-v2}, respectively; Adam ($\text{lr}{=}10^{-4}$) with linear warmup; $10$ evaluation episodes per checkpoint; $3$ random seeds.

\paragraph{EWA\textendash VQ\textendash ODT hyperparameters.}
Attention bias scale $\beta{=}0.05$; decay $\phi{=}0.05$; chosen weight $\delta{=}0.8$; codebook size $M{=}27$; grid bins chosen adaptively per dimension. 

\paragraph{Evaluation protocol.}
At each checkpoint we compute per-seed metrics; we then average across seeds at each step, and finally average those per-step means across all training steps to obtain a single scalar per (environment, algorithm, metric). Improvements are reported as relative change vs.\ ODT.

\subsection{Metrics}
\label{subsec:metrics}

We follow the ODT codebase and report seven metrics:
\begin{itemize}
    \item \texttt{evaluation/return\_mean\_gm}: geometric mean of evaluation returns
    \item \texttt{evaluation/return\_std\_gm}: variability of evaluation returns
    \item \texttt{evaluation/return\_vs\_samples}: return as a function of consumed environment steps
    \item \texttt{evaluation/length\_mean\_gm}: geometric mean of evaluation episode lengths
    \item \texttt{evaluation/length\_std\_gm}: variability of episode lengths
    \item \texttt{aug\_traj/return}: return of trajectories generated during fine-tuning
    \item \texttt{aug\_traj/length}: length of trajectories generated during fine-tuning
\end{itemize}

\subsection{Main Results}
\label{subsec:main-results}

\begin{figure}[t]
  \centering
  \includegraphics[width=0.92\linewidth]{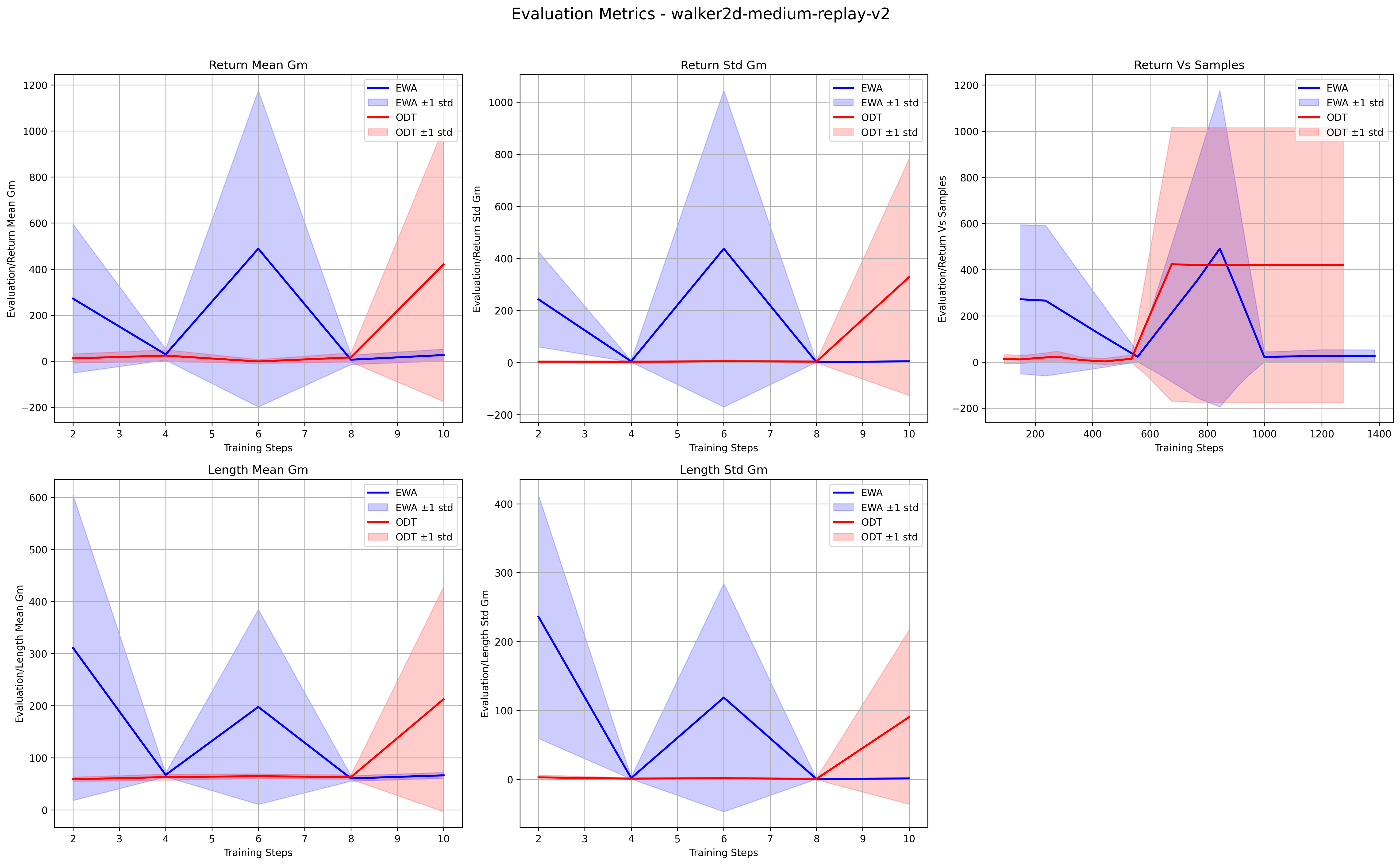}
  \vspace{0.3cm} 

  \includegraphics[width=0.92\linewidth]{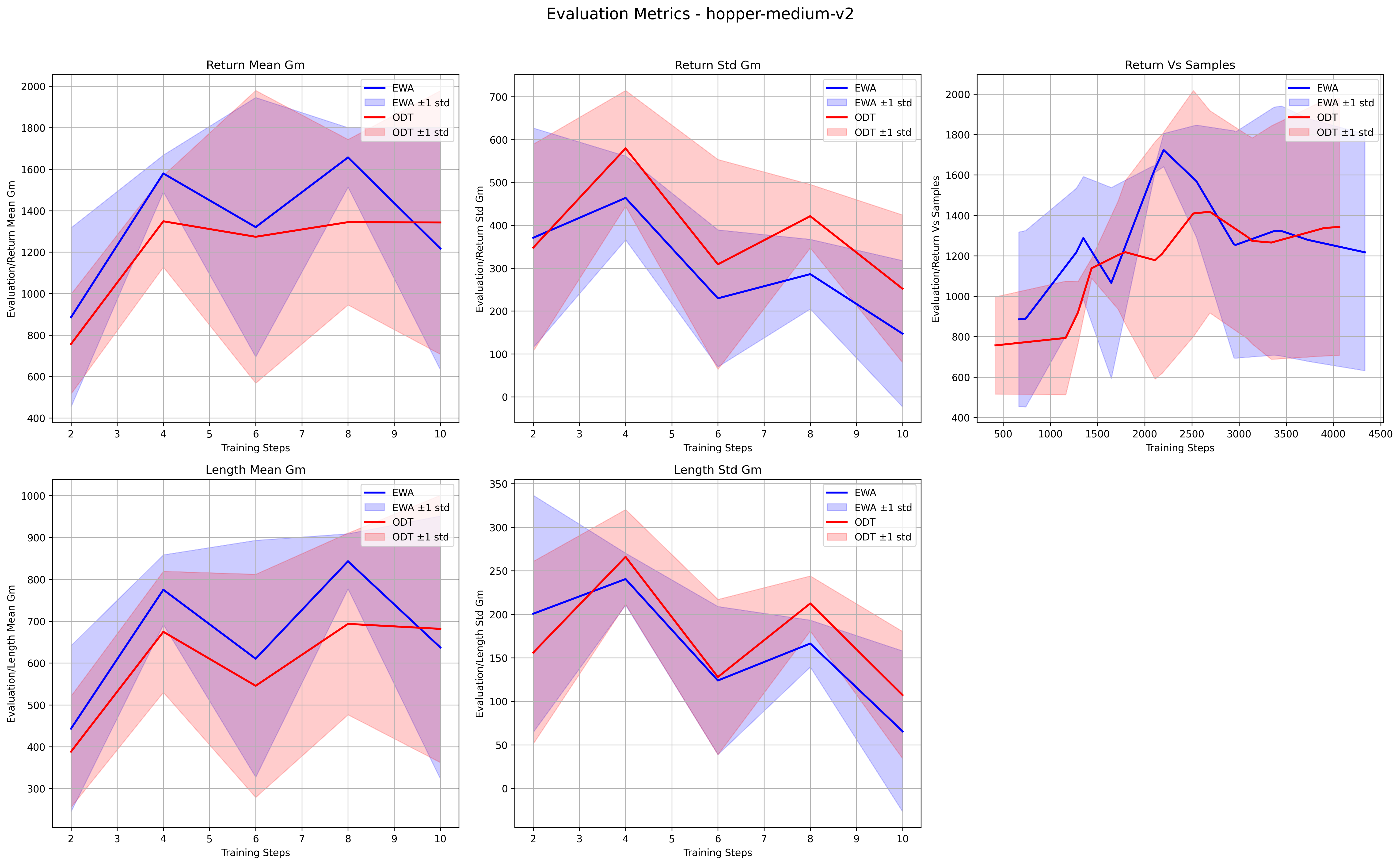}

  \caption{Evaluation curves. 
  Top: \texttt{walker2d-medium-replay-v2}, where EWA--VQ--ODT substantially improves \emph{Return Mean Gm}, \emph{Length Mean Gm}, and \emph{Return vs Samples}, with higher mid-training variance reflecting volatility in the higher-dimensional action space. 
  Bottom: \texttt{hopper-medium-v2}, where gains are modest but consistent: EWA--VQ--ODT raises \emph{Return Mean Gm}, \emph{Length Mean Gm}, and \emph{Return vs Samples}, while reducing variance for more stable policies. 
  Shaded bands show $\pm 1$ geometric std across seeds.}
  \label{fig:eval-results}
\end{figure}

\begin{table}[t]
\centering
\small
\caption{Summary of EWA--VQ--ODT vs.\ ODT on \texttt{walker2d-medium-replay-v2}. 
Averages are over seeds 1--3 and training steps (gm = geometric mean). 
A \checkmark indicates improvement is desirable.}
\label{tab:results-walker}
\begin{tabular}{lrrr}
\toprule
\textbf{Metric} & \textbf{EWA--VQ--ODT Avg} & \textbf{ODT Avg} & \textbf{$\Delta$ vs.\ ODT (\%)} \\
\midrule
evaluation/return\_mean\_gm    & 167.1 & 82.4  & +102.9\% \checkmark \\
evaluation/return\_std\_gm     & 155.3 & 57.4  & +170.6\% \(\times\) \\
evaluation/return\_vs\_samples & 179.3 & 123.4 & +45.3\% \checkmark \\
evaluation/length\_mean\_gm    & 134.5 & 87.4  & +53.9\% \checkmark \\
evaluation/length\_std\_gm     & 70.1  & 16.2  & +331.7\% \(\times\) \\
aug\_traj/return               & 87.5  & 64.0  & +36.7\% \checkmark \\
aug\_traj/length               & 116.1 & 90.5  & +28.3\% \checkmark \\
\bottomrule
\end{tabular}
\end{table}

\begin{table}[t]
\centering
\small
\caption{Summary of EWA--VQ--ODT vs.\ ODT on \texttt{hopper-medium-v2}. 
Averages are over seeds 1--3 and training steps (gm = geometric mean).}
\label{tab:results-hopper}
\begin{tabular}{lrrr}
\toprule
\textbf{Metric} & \textbf{EWA--VQ--ODT Avg} & \textbf{ODT Avg} & \textbf{$\Delta$ vs.\ ODT (\%)} \\
\midrule
evaluation/return\_mean\_gm    & 1343.9 & 1254.8 & +7.1\% \checkmark \\
evaluation/return\_std\_gm     & 307.9  & 403.3  & $-23.6$\% \checkmark \\
evaluation/return\_vs\_samples & 1325.1 & 1222.1 & +8.4\% \checkmark \\
evaluation/length\_mean\_gm    & 680.6  & 629.0  & +8.2\% \checkmark \\
evaluation/length\_std\_gm     & 162.3  & 187.2  & $-13.3$\% \checkmark \\
aug\_traj/return               & 867.9  & 884.5  & $-1.9$\% \(\times\) \\
aug\_traj/length               & 357.5  & 376.8  & $-5.1$\% \(\times\) \\
\bottomrule
\end{tabular}
\end{table}

Across tasks, curves (Figs.~\ref{fig:eval-results}) shows that EWA--VQ--ODT accelerates learning and improves sample efficiency (\texttt{return\_vs\_samples}). 
Tables~\ref{tab:results-walker}--\ref{tab:results-hopper} confirm that on \texttt{walker2d-medium-replay-v2} it more than doubles average returns and improves length-based metrics, albeit with higher variance. 
On \texttt{hopper-medium-v2}, improvements are smaller but stable, with reduced variance and more consistent learning.

\paragraph{Discussion.}
Results support the hypothesis that a recency-weighted per-action memory provides a useful inductive bias: it accelerates learning and, in lower-dimensional tasks (Hopper), stabilizes outcomes. 
In higher dimensions (Walker2d), the mechanism can over-reinforce frequent codes, motivating refinements such as per-head bias scaling, reward normalization, and codebook adjustments.

\paragraph{Compute Overhead.}
EWA--VQ--ODT adds only a GPU-resident vector of attractions and a bias to attention logits. 
Routing is $O(d)$ via grid lookup; updates are vectorized indexed adds. 
Training-time overhead is negligible relative to ODT.


\section{Conclusion}
\label{sec:conclusion}

We introduced \textbf{EWA--VQ--ODT}, a minimal, cognitively inspired add–on that equips Decision/Online Decision Transformers with an explicit per–action outcome memory. Continuous actions are routed to a compact vector–quantized codebook via direct grid lookup; each code maintains a decayed attraction updated online by a simplified EWA rule. We inject these attractions as a small bias on action–token attention columns, leaving the backbone, objective, and training loop unchanged. On standard continuous–control benchmarks, EWA--VQ--ODT improves average return and sample efficiency over a strong ODT baseline, while adding negligible computational overhead and offering interpretable diagnostics via attraction traces.

Beyond empirical gains, we provided two simple theoretical analyses: a closed–form characterization of the attraction process establishing boundedness and an expected steady state, and a softmax perturbation bound that constrains attention drift when the bias is clipped. These results justify stable scaling of the decay/bias parameters and support safe integration into transformer attention.

\paragraph{Limitations.}
Gains are environment–dependent: higher–dimensional action spaces can exhibit larger mid–training variance or late–training regressions if the bias scale or codebook granularity are mis–tuned. A fixed codebook trades similarity generalization for routing speed, and hard assignments introduce quantization error. Our study focuses on simulator benchmarks; human–subject evaluations and broader domains are left for future work.

\paragraph{Outlook.}
Promising directions include (i) per–head or token–wise learned bias scales and normalization of the attraction signal; (ii) learned or hierarchical codebooks and soft/state–conditioned routing; (iii) coupling EWA–style memory with value– or model–based credit assignment; and (iv) expanded evaluation (more seeds, tasks, and real–world settings). Overall, explicit, low–dimensional memories of action “success’’ and “failure’’ provide a practical inductive bias for transformer–based RL, complementing content–based attention without complicating the architecture.

\bibliographystyle{plainnat}  
\bibliography{main}

\clearpage
\appendix
\section{Proofs of Theoretical Results}

\subsection{Proof of Proposition~\ref{prop:ewa_closedform}}
\label{app:ewa_proof}

\textbf{Proposition~\ref{prop:ewa_closedform}} [Closed form, boundedness, steady state]

Let the per-code attraction follow the simplified EWA update
\[
A_\ell(t) = (1-\phi) A_\ell(t-1) + \delta\,\tilde r_t\,\mathbf{1}\{\ell=j_t\}, 
\quad \phi \in (0,1), \quad |\tilde r_t|\le R .
\]
Then $A_\ell(t)$ admits a closed form, is uniformly bounded by 
$\delta R/\phi$, and under stationarity satisfies
\[
\lim_{t \to \infty} \, \mathbb{E}[A_\ell(t)] \;=\; \frac{\delta\,p_\ell\,\mu_\ell}{\phi}.
\]

\begin{proof}

As described in the main text, $A_\ell(t)$: the “attraction” (memory) for action code $\ell$ after step $t$.
Update rule (simplified EWA) is:
  \[
  A_\ell(t)=(1-\phi)\,A_\ell(t-1)+\delta\,\tilde r_t\,\mathbf{1}\{\ell=j_t\},
  \]
where $0<\phi<1$ is the decay rate, $\delta$ scales new rewards, $\tilde r_t$ is bounded with $|\tilde r_t|\le R$, $j_t$ is the code chosen at time $t$, and $\mathbf{1}\{\ell=j_t\}$ is the indicator of choosing $\ell$.

Every step shrinks all memories by $(1-\phi)$ and adds $\delta\tilde r_t$ only to the chosen code’s memory.

\subsection*{Step 1: Closed form}
Unrolling the recursion gives:
\[
\begin{aligned}
A_\ell(t)
&=(1-\phi)A_\ell(t-1)+\delta\,\tilde r_t\,\mathbf{1}\{\ell=j_t\}\\
&=(1-\phi)\Big[(1-\phi)A_\ell(t-2)+\delta\,\tilde r_{t-1}\,\mathbf{1}\{\ell=j_{t-1}\}\Big]+\delta\,\tilde r_t\,\mathbf{1}\{\ell=j_t\}\\
&=(1-\phi)^2A_\ell(t-2)+(1-\phi)\delta\,\tilde r_{t-1}\mathbf{1}\{\ell=j_{t-1}\}+\delta\,\tilde r_t\mathbf{1}\{\ell=j_t\}\\
&\;\;\vdots\\
&=(1-\phi)^t A_\ell(0)+\delta\sum_{\tau=1}^{t}(1-\phi)^{\,t-\tau}\,\tilde r_\tau\,\mathbf{1}\{\ell=j_\tau\}.
\end{aligned}
\]
It demonstrates that $A_\ell(t)$ is a geometrically decayed sum of past realized rewards for code $\ell$.

 \subsection*{Step 2: Uniform boundedness}

We start from the unrolled form, which is already established:
\[
A_\ell(t) \;=\; (1-\phi)^t A_\ell(0)\;+\;\delta\sum_{\tau=1}^{t}(1-\phi)^{t-\tau}\,\tilde r_\tau\,\mathbf{1}\{\ell=j_\tau\}.
\]

First, apply the triangle inequality and absolute values to separate the two terms:
\[
|A_\ell(t)| \;\le\; |(1-\phi)^t A_\ell(0)| \;+\; \delta \left|\sum_{\tau=1}^{t}(1-\phi)^{t-\tau}\,\tilde r_\tau\,\mathbf{1}\{\ell=j_\tau\}\right|.
\]
This step follows because $|x+y|\le|x|+|y|$.

Next, pull the absolute value inside the sum:
\[
|A_\ell(t)| \;\le\; (1-\phi)^t |A_\ell(0)| \;+\; \delta \sum_{\tau=1}^{t}(1-\phi)^{t-\tau}\,|\tilde r_\tau|\,\mathbf{1}\{\ell=j_\tau\}.
\]
This uses the fact that $\big|\sum u_\tau\big|\le\sum |u_\tau|$. Also, $|(1-\phi)^{t-\tau}|=(1-\phi)^{t-\tau}$ since $(1-\phi)\in(0,1)$, and $|\mathbf{1}\{\ell=j_\tau\}|=\mathbf{1}\{\ell=j_\tau\}$ since the indicator is always 0 or 1.

Now, apply the reward bound and indicator bound. Since $|\tilde r_\tau|\le R$ and $\mathbf{1}\{\ell=j_\tau\}\le 1$, each term can be bounded as follows:
\[
|A_\ell(t)| \;\le\; (1-\phi)^t |A_\ell(0)| \;+\; \delta \sum_{\tau=1}^{t}(1-\phi)^{t-\tau}\,R\,\mathbf{1}\{\ell=j_\tau\}
\;\le\; (1-\phi)^t |A_\ell(0)| \;+\; \delta R \sum_{\tau=1}^{t}(1-\phi)^{t-\tau}.
\]

Reindex the geometric sum by setting $k=t-\tau$. As $\tau$ goes from $1$ to $t$, $k$ goes from $t-1$ down to $0$. Thus
\[
\sum_{\tau=1}^{t}(1-\phi)^{t-\tau} \;=\; \sum_{k=0}^{t-1}(1-\phi)^k.
\]
Hence
\[
|A_\ell(t)| \;\le\; (1-\phi)^t |A_\ell(0)| \;+\; \delta R \sum_{k=0}^{t-1}(1-\phi)^k.
\]

Finally, upper-bound the finite geometric series by the infinite one:
\[
\sum_{k=0}^{t-1}(1-\phi)^k \;\le\; \sum_{k=0}^{\infty}(1-\phi)^k \;=\; \frac{1}{\phi}.
\]
This holds because $0<1-\phi<1$ and all terms are nonnegative, so extending the sum can only increase it. Therefore
\[
|A_\ell(t)| \;\le\; (1-\phi)^t |A_\ell(0)| \;+\; \frac{\delta R}{\phi}.
\]

In summary, the general uniform bound, without assuming a particular initialization, is
\[
|A_\ell(t)| \;\le\; (1-\phi)^t |A_\ell(0)| \;+\; \frac{\delta R}{\phi}
\;\le\; |A_\ell(0)| \;+\; \frac{\delta R}{\phi},
\]
since $(1-\phi)^t\le 1$. If we initialize $A_\ell(0)=0$ as in Algorithm~1, the bound simplifies to
\[
|A_\ell(t)| \;\le\; \frac{\delta R}{\phi} \qquad \text{for all } t.
\]
Even if $A_\ell(0)\neq 0$, the extra term $(1-\phi)^t|A_\ell(0)|$ decays to zero, so the bound tightens to $\delta R/\phi$ over time.

\subsection*{Step 3: Steady state (long-run expectation)}

We begin again from the one-step update:
\[
A_\ell(t)=(1-\phi)A_\ell(t-1)+\delta\,\tilde r_t\,\mathbf{1}\{\ell=j_t\}.
\]

Taking expectations on both sides and applying linearity gives
\[
\mathbb{E}[A_\ell(t)]
=(1-\phi)\,\mathbb{E}[A_\ell(t-1)]
+\delta\,\mathbb{E}\big[\tilde r_t\,\mathbf{1}\{\ell=j_t\}\big].
\]

Now compute $\mathbb{E}[\tilde r_t\,\mathbf{1}\{\ell=j_t\}]$ by conditioning on $j_t$:
\[
\mathbb{E}[\tilde r_t\,\mathbf{1}\{\ell=j_t\}]
=\mathbb{E}\big[\; \mathbb{E}[\tilde r_t\,\mathbf{1}\{\ell=j_t\}\mid j_t]\;\big]
=\mathbb{E}\big[\mathbf{1}\{\ell=j_t\}\,\mathbb{E}[\tilde r_t\mid j_t]\big].
\]

By the definitions under stationarity, let $p_\ell=\Pr(j_t=\ell)$ and $\mu_\ell=\mathbb{E}[\tilde r_t\mid j_t=\ell]$. Then
\[
\mathbb{E}[\tilde r_t\,\mathbf{1}\{\ell=j_t\}]
=\sum_{j} \Pr(j_t=j)\,\mathbf{1}\{\ell=j\}\,\mathbb{E}[\tilde r_t\mid j_t=j]
=p_\ell \mu_\ell.
\]

Substituting back, the recursion for expectations becomes
\[
\mathbb{E}[A_\ell(t)]
=(1-\phi)\,\mathbb{E}[A_\ell(t-1)]\;+\;\delta\,p_\ell\,\mu_\ell.
\]

Letting $y_t:=\mathbb{E}[A_\ell(t)]$, $\alpha:=1-\phi\in(0,1)$, and $c:=\delta\,p_\ell\,\mu_\ell$, the recursion is
\[
y_t=\alpha\,y_{t-1}+c.
\]

The general solution to this linear difference equation is
\[
y_t=\alpha^t y_0 + \frac{c}{1-\alpha}\bigl(1-\alpha^t\bigr).
\]

Since $\alpha=1-\phi\in(0,1)$, we have $\alpha^t\to 0$ as $t\to\infty$. Therefore
\[
\lim_{t \to \infty} \, \mathbb{E}[A_\ell(t)] \;=\; \frac{c}{1-\alpha}
=\frac{\delta\,p_\ell\,\mu_\ell}{\phi}.
\]

Equivalently, we can check the fixed point directly: setting $y_t=y_{t-1}=y_\star$ in the recursion yields
\[
y_\star=\alpha y_\star+c,
\]
whose solution is $y_\star=\tfrac{c}{1-\alpha}=\tfrac{\delta p_\ell \mu_\ell}{\phi}$. Since $|\alpha|<1$, the recursion is a contraction and converges to this unique fixed point.

\end{proof}

\subsection{Proof of Lemma~\ref{lem:softmax_bound}}
\label{app:softmax_proof}

\textbf{Lemma~\ref{lem:softmax_bound}}[Controlled attention drift under a bounded bias]

Let $p=\mathrm{softmax}(z)\in\Delta^n$ and $q=\mathrm{softmax}(z+b)$ for any bias $b$ with $\|b\|_\infty\le \varepsilon$.  
Then the total variation distance satisfies
\[
\|q-p\|_{\mathrm{TV}} \;\le\; \tanh(\varepsilon).
\]

\begin{proof}
We begin by recalling definitions. The softmax distribution is given by $p_i=\frac{e^{z_i}}{\sum_j e^{z_j}}$ and $q_i=\frac{e^{z_i+b_i}}{\sum_j e^{z_j+b_j}}$. The total variation (TV) distance is $\|q-p\|_{\mathrm{TV}}=\tfrac12\sum_i |q_i-p_i|$. This measures the discrepancy between two probability distributions, taking values in $[0,1]$.

To analyze this bound, we first express $q$ in terms of $p$. Multiplying and dividing by $\sum_k e^{z_k}$, one obtains
\[
q_i=\frac{e^{z_i+b_i}}{\sum_j e^{z_j+b_j}}
=\frac{e^{b_i}}{\sum_j p_j e^{b_j}}\,p_i.
\]
Letting $Z=\sum_j p_j e^{b_j}$ and $\lambda_i=e^{b_i}/Z$, this simplifies to $q_i=\lambda_i p_i$. Furthermore, $\sum_i p_i\lambda_i=1$, so the $p$-weighted average of the reweighting factors $\lambda_i$ is exactly one.

Next, we bound the reweighting factors. Since $b_i\in[-\varepsilon,\varepsilon]$, it follows that $e^{b_i}\in[e^{-\varepsilon},e^{\varepsilon}]$. The denominator $Z$ is a convex combination of such terms, so $Z\in[e^{-\varepsilon},e^{\varepsilon}]$. Therefore
\[
\lambda_i=\frac{e^{b_i}}{Z}\in\left[\frac{e^{-\varepsilon}}{e^{\varepsilon}},\frac{e^{\varepsilon}}{e^{-\varepsilon}}\right]
=[e^{-2\varepsilon},e^{2\varepsilon}],
\]
so each likelihood ratio $q_i/p_i$ is bounded between $a=e^{-2\varepsilon}$ and $b=e^{2\varepsilon}$ with $ab=1$.

Now we express the TV distance in terms of these reweighting factors:
\[
\|q-p\|_{\mathrm{TV}}=\frac{1}{2}\sum_i |q_i-p_i|
=\frac{1}{2}\sum_i p_i\,|\lambda_i-1|.
\]
Partitioning indices into $S=\{i:\lambda_i\ge1\}$ and $S^c$, and using the fact that $\sum_i p_i\lambda_i=1$ implies $\sum_i p_i(\lambda_i-1)=0$, we deduce that
\[
\|q-p\|_{\mathrm{TV}}=\sum_{i\in S}p_i(\lambda_i-1)=\sum_{i\in S^c}p_i(1-\lambda_i).
\]
Thus TV reduces to the total “mass shifted upwards” (or downwards), and both are equal.

To maximize this deviation under the constraints $\lambda_i\in[a,b]$ and $\sum_i p_i\lambda_i=1$, one only needs to consider extreme points where each $\lambda_i$ equals either $a$ or $b$. The worst case occurs with a two-point mixture: set $\lambda=b$ on some fraction $s$ of the $p$-mass and $\lambda=a$ on the rest. Enforcing $\sum_i p_i\lambda_i=1$ yields $1=sb+(1-s)a$, hence $s=\frac{1-a}{b-a}$. Substituting into the expression for TV gives
\[
\|q-p\|_{\mathrm{TV}}
=\frac{1}{2}\big(s(b-1)+(1-s)(1-a)\big)
=\frac{(1-a)(b-1)}{b-a}.
\]

Finally, substituting $a=e^{-2\varepsilon}$ and $b=e^{2\varepsilon}$, we compute
\[
\|q-p\|_{\mathrm{TV}}
=\frac{b-1}{b+1}
=\frac{e^{2\varepsilon}-1}{e^{2\varepsilon}+1}
=\tanh(\varepsilon).
\]

This establishes the desired bound. Intuitively, the perturbation $b$ shifts logits by at most $\pm\varepsilon$, so each attention probability can change by a factor in $[e^{-2\varepsilon},e^{2\varepsilon}]$. The worst-case change occurs when some entries receive the full amplification while others receive the full suppression, balanced to maintain normalization. This yields the clean $\tanh(\varepsilon)$ formula. For small $\varepsilon$, the bound behaves linearly, $\tanh(\varepsilon)\approx\varepsilon$, so clipping the bias ensures the drift remains small and interpretable.
\end{proof}

\section{Hyperparameters for \textsc{EWA--VQ--ODT}}
\label{app:hyperparams}

Unless noted, all runs use the same backbone and optimizer as the ODT baseline. The values below reflect the configurations used in our main experiments.

\subsection{Core EWA--VQ Settings}
\begin{table}[h]
\centering
\small
\begin{tabular}{lcl}
\toprule
\textbf{Parameter} & \textbf{Value} & \textbf{Description} \\
\midrule
Attraction decay $\phi$ & $0.05$ & Per–step global decay: $A\!\leftarrow\!(1-\phi)A$ \\
Chosen weight $\delta$ & $0.8$ & Additive update for routed code: $A_{j_t}\!+\!=\!\delta\,\tilde r_t$ \\
Attention bias scale $\beta$ & $0.05$ & Scales the attraction bias added to action columns \\
Trajectory length (max) & $1000$ & Max steps processed per trajectory \\
VQ codebook size $M$ & $27$ & Number of discrete action surrogates (codes) \\
Grid bins factor & $1.0$ & Adaptive bins per action dim (see below) \\
Routing & direct lookup & Grid cell $\!\to\!$ code index via cached table $\mathsf{T}$ \\
\bottomrule
\end{tabular}
\end{table}

\paragraph{Codebook/grid construction.}
We use an axis-aligned grid with $B$ bins per dimension; for the tasks studied we set $B{=}3$. The offline table $\mathsf{T}:\{0,\ldots,B^d\!-\!1\}\!\to\!\{1,\ldots,M\}$ maps each grid cell to its nearest code. The table is cached per $(d,B,M,\text{env})$ and reused.

\subsection{Environment–Specific Settings}
\begin{table}[h]
\centering
\small
\begin{tabular}{lcccccc}
\toprule
\textbf{Env} & \textbf{Action dim} & \textbf{$B$} & \textbf{Cells $B^d$} & \textbf{$M$ used} & \textbf{Online RTG} & \textbf{Eval RTG} \\
\midrule
\texttt{hopper-medium-v2} & 3 & 3 & $27$ & $27$ (full coverage) & $7200$ & $3600$ \\
\texttt{walker2d-medium-replay-v2} & 6 & 3 & $729$ & $27$ (subset) & $10000$ & $5000$ \\
\bottomrule
\end{tabular}
\end{table}

\begin{table}[h]
\centering
\small
\begin{tabular}{lccc}
\toprule
\textbf{Env} & \textbf{Eval context length} & \textbf{Ordering (pos.\ emb.)} & \textbf{Notes} \\
\midrule
\texttt{hopper-medium-v2} & $20$ & $1$ & Positional embeddings enabled \\
\texttt{walker2d-medium-replay-v2} & $5$ & $0$ & No positional embeddings \\
\bottomrule
\end{tabular}
\end{table}

\paragraph{Optional configuration (not in main results).}
For 6D tasks like \texttt{halfcheetah-medium-expert-v2}, we use the same $B{=}3$ (729 cells), $M{=}27$ routed codes, and (online/eval) RTGs of $(12000, 6000)$.

\subsection{Training and Model Hyperparameters}
\begin{table}[h]
\centering
\small
\begin{tabular}{lcl}
\toprule
\textbf{Training} & \textbf{Value} & \textbf{Details} \\
\midrule
Online iterations (quick runs) & $10$ & Evaluate every $2$ iterations \\
Updates per iteration & $30$ & Per online iteration \\
Batch size & $64$ & Mini-batch size \\
Eval episodes / checkpoint & $10$ & Parallel eval environments \\
Seeds & $3$ & Reported as averages \\
\bottomrule
\end{tabular}
\end{table}

\begin{table}[h]
\centering
\small
\begin{tabular}{lcl}
\toprule
\textbf{Backbone/Opt} & \textbf{Value} & \textbf{Details} \\
\midrule
Embedding dim & $512$ & Token embedding size \\
Transformer layers & $4$ & Encoder blocks \\
Attention heads & $4$ & Multi-head self-attention \\
Activation & ReLU & \\
Dropout & $0.1$ & Residual/MLP dropout \\
Initial temperature & $0.1$ & As in code (for logits scaling) \\
Optimizer & Adam & \\
Learning rate & $1\times 10^{-4}$ & With warmup \\
Weight decay & $5\times 10^{-4}$ & \\
Warmup steps & $10{,}000$ & Linear warmup schedule \\
Device & CUDA & GPU training \\
\bottomrule
\end{tabular}
\end{table}

\paragraph{Grid adaptation rules.}
For $d\!\le\!6$, we use $B{=}3$ bins/dim. (General logic: $B\!\in\![2,8]$ with a cap on $M$—here $M{=}27$; max $M$ can be raised up to $128$ if memory allows.) The offline $\mathsf{T}$ table is built once (GPU-batched $L_2$) and cached.

\paragraph{Implementation note.}
Attractions $A\!\in\!\mathbb{R}^M$ are stored on GPU; each step applies one vectorized decay and a single indexed add for the routed code. The attention bias adds $\beta\,A$ only to action-token columns prior to masking and softmax.

\end{document}